\def\ICLRNAMED{}
\documentclass{article} 
\usepackage{iclr2027_conference,times}
\usepackage{iftex}
\ifPDFTeX
  \usepackage[T1]{fontenc}
\else
  \usepackage{fontspec}
  \IfFontExistsTF{Times New Roman}
    {\setmainfont{Times New Roman}}
    {\setmainfont{TeX Gyre Termes}}
\fi

\usepackage{booktabs}
\usepackage{array}
\usepackage{tabularx}
\usepackage{multirow}
\newcolumntype{Y}{>{\raggedright\arraybackslash}X}
\newcolumntype{C}{>{\centering\arraybackslash}X}
\usepackage{amsmath,amssymb,amsthm}
\usepackage{graphicx}
\usepackage{url}
\usepackage{placeins}
\usepackage{float}
\usepackage{enumitem}
\setlist[itemize]{leftmargin=0pt, labelsep=0.35em, topsep=0.25\baselineskip,
  itemsep=0.25\baselineskip, parsep=0pt, partopsep=0pt, itemindent=0pt,
  listparindent=0pt}
\theoremstyle{plain}
\newtheorem{proposition}{Proposition}[section]

\theoremstyle{remark}

\title{Radioactive Feedback: Tracing Models Trained by \mbox{Label-Only} LLM Supervision}

\ifdefined\ICLRNAMED
\iclrfinalcopy
\usepackage{hyperref}
\hypersetup{hidelinks}
\makeatletter
\def\@maketitle{%
  \vbox{\hsize\textwidth
    {\LARGE\sc \@title\par}%
    \vskip 0.85em
    \def\@makefnmark{\hbox{$^{\@thefnmark}$}}%
    {\centering
      Yuejia Li$^{1}$\thanks{\begin{tabular}[t]{@{}l@{}}
        Correspondence to Yuejia Li (\texttt{liyuejia@whu.edu.cn}).\\
        Code: \url{https://github.com/Kemalau/RadioactiveFeedBack}\\
        Project: \url{https://github.com/Kemalau/OneLinetoProtectYourReward}
      \end{tabular}},
      JunHeng Li$^{1}$, Hao Wu$^{2}$, Ke He$^{1}$, Mang Ye$^{1}$\par
      \vskip 0.35em
      {\small $^{1}$Wuhan University\qquad $^{2}$Southeast University\par}}%
    \vskip 0.22in minus 0.05in}}
\makeatother
\else
\author{Anonymous authors}
\fi

\begin{document}

\maketitle
\ifdefined\ICLRNAMED
  \fancyhead{}%
  \renewcommand{\headrulewidth}{0pt}%
\fi
\suppressfloats[t]

\begin{abstract}
Conventional distillation transfers capability through teacher-generated
demonstrations.  LLM-as-a-Judge knowledge distillation offers a distinct
and increasingly practical alternative.  In the label-only variant
studied here, the student generates every response and improves from
scalar scores alone.  Prior work shows that this non-demonstrative
feedback can improve student capability, but it creates a fundamentally
different provenance problem.  No teacher-generated answer enters the
student's training set, so any evaluator-specific trace must pass through
a low-bandwidth reward signal rather than copied text.  Existing
watermarks and distillation detectors do not cover this score-only
channel.  We introduce Radioactive Feedback, which encodes a registered
pattern of behavioral preferences in the rewards themselves.  The
evaluator applies bounded pointwise adjustments to quality-qualified
responses, allowing repeated optimization to transfer the pattern to a
student trained by another party.  Experiments on code and mathematics
show that the registered preferences transfer through scores alone and
across datasets, enabling black-box detection against a fixed reference
pool.
\end{abstract}

\section{Introduction}

Conventional supervised distillation gives a student complete teacher
responses and trains it to imitate them token by token.  LLM-as-a-Judge
knowledge distillation changes both the object of supervision and its
information channel.  The broader paradigm can use critiques,
preferences, or scores; in the label-only variant studied here, the
student generates candidate responses from its current policy while a
provider-controlled evaluator returns only scalar judgments.  Existing
RLAIF and on-policy distillation research shows that this feedback can produce capability
gains~\citep{bai2022constitutional,lee2023rlaif,yuan2024selfrewarding,fang2026ropd}.
The teacher never demonstrates how to answer; optimization instead
transfers aspects of the evaluator's preferences through judgments of
the student's own outputs.  We call this setting \emph{label-only
evaluator distillation}.

This low-bandwidth channel creates a distinct provenance problem.  No teacher-generated text ever enters the student's
training set, so a detectable evaluator-specific trace cannot be
explained by copying watermarked tokens.  It must survive through scalar
feedback alone.  We therefore ask whether evaluator information can
enter the student's behavior through scores and remain recognizable.

Existing provenance methods rely on higher-bandwidth objects.  Text
watermarks and their radioactive variants transmit provenance through
teacher-generated text~\citep{kirchenbauer2023watermark,sander2024radioactive}.
Claude's text watermark is a version of SynthID-Text~\citep{anthropic2026watermark,dathathri2024synthid}.
Antidistillation sampling likewise modifies teacher
tokens~\citep{savani2025antidistillation}.  PreferCare marks a fixed
preference dataset~\citep{lou2025prefercare}.  Model fingerprints mark
weights, trigger behaviors, or training text and can be used when the marked model
or its training text is available~\citep{uchida2017embedding,xu2024instructional,gloaguen2026fingerprinting}.
These mechanisms do not write an attribution signal into scalar
feedback, where the evaluator can shape only the score
(Figure~\ref{fig:motivation}).

\begin{figure}[t]
    \centering
    \includegraphics[width=\textwidth]{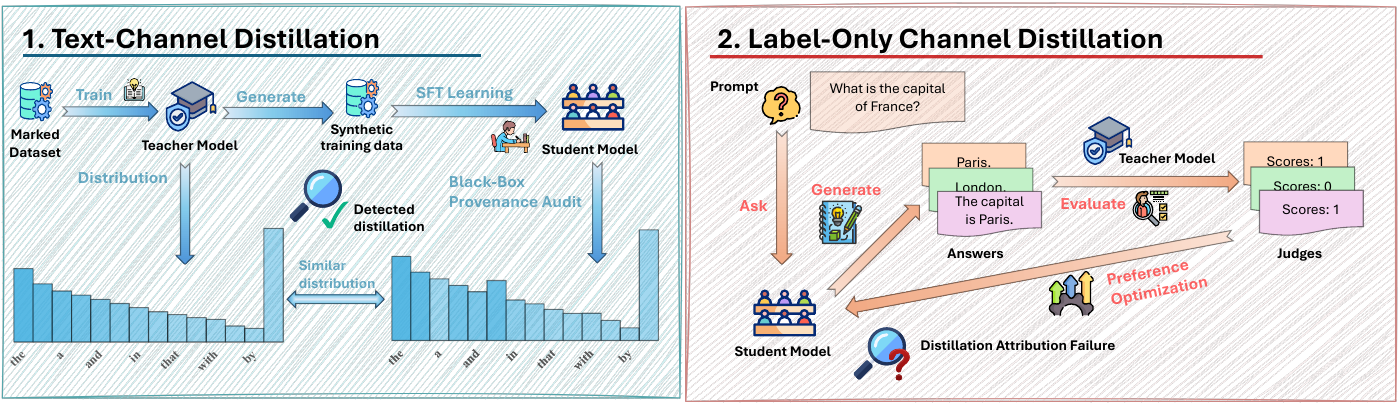}
    \caption{Text distillation and label-only feedback.}
    \label{fig:motivation}
\end{figure}

To address this problem, we introduce \emph{Radioactive Feedback}, a method for embedding a
registered behavioral code in scalar feedback.
The code selects preferred alternatives on a set of response features,
or \emph{carriers}.  When an individual response meets a quality floor
and exposes a classifiable carrier value, the evaluator applies a
bounded pointwise reward adjustment in the registered direction.
Repeated optimization can then transfer these preferences to the student.

Carriers range from lightweight response conventions, such as formatting
choices, to deeper solution strategies, such as alternative dynamic-programming
or geometric approaches.  This allows the registered code to span both surface
form and problem-solving behavior.  Training and auditing on different datasets
tests whether the resulting pattern generalizes beyond the training set.

At audit time, the provider measures alignment in student
responses and compares it with a reference pool fixed before scoring,
thereby testing whether the student distilled the enrolled teacher.

This paper makes three contributions.
\begin{itemize}
\item \textbf{Provenance across non-demonstrative feedback.}
We formalize tracing when the student generates every training
response and the evaluator returns only scalar scores.  No
teacher-generated text enters the training set, and the secret
belongs to the feedback provider rather than the trainer of the
audited model.
\item \textbf{Radioactive Feedback.}
A keyed, bounded, quality-gated pointwise tilt on eligible
responses distributes a registered code over multiple carriers,
covering answer form and solution methods, and supports black-box
detection from student responses.
\item \textbf{Five experimental evaluations with supporting ablations.}
Single-carrier and multi-carrier experiments establish transfer and a
joint detector signal; registered-key attribution tests code
specificity; large-scale detection adds held-out cross-family
controls; persistence covers post-training and removal, plus a
$\rho$ ablation (appendix).
\end{itemize}

Experiments with Llama and Qwen students show strong watermark effects
on code and mathematics, with single-carrier alignment gains up to
$64.9$\,pp.  Joint four-carrier training remains positive across all nine evaluated
student--evaluator combinations, ranging from $+7.6$ to $+15.4$\,pp
(Table~\ref{tab:e3}).  The watermark also transfers from TACO and
OpenR1-Math to multiple held-out code and mathematics benchmarks
(Table~\ref{tab:e2}).  In the two settings used for calibrated auditing,
all $8/8$ independently trained marked checkpoints are detected at nominal
$\alpha=1\%$ (Table~\ref{tab:detector-ops}), and all $40$ held-out cross-family controls remain below
their respective detection thresholds (Table~\ref{tab:cross-family-e4}).

\section{Related Work}

\paragraph{LLM-as-a-Judge knowledge distillation.}
LLM evaluators increasingly supervise students through critiques,
preferences, and scalar judgments~\citep{bai2022constitutional,lee2023rlaif,zheng2023judging,
yuan2024selfrewarding}.  Such labels can drive RL or direct preference
optimization without exposing teacher logits~\citep{rafailov2023direct}.
We study the label-only case in which each student rollout receives only
a scalar score.
Rubric-based On-policy Distillation (ROPD) similarly scores on-policy rollouts with prompt-specific rubrics,
although teacher-generated references help construct those
rubrics~\citep{fang2026ropd}.
Shen, Tu, and Wang likewise reward the student's own rollouts from an LLM judge, computing the continuous reward from the judge's Yes/No token logits~\citep{shen2026rlkd}.
These works ask how evaluator feedback
improves a student; we ask whether that influence can be identified
when the student generated every response.

\paragraph{Watermarking supervision channels.}
Text watermarks modify generated-token statistics and radioactive
watermarks detect their residue after training on marked synthetic
text~\citep{kirchenbauer2023watermark,kuditipudi2024robust,
zhao2024provable,christ2024undetectable,sander2024radioactive}.
Pan et al.\ study inheritance and removal of radioactive watermarks under distillation, still carried by teacher-generated text~\citep{pan2025llmwatermark}.
Antidistillation sampling acts on the teacher's token
distribution~\citep{savani2025antidistillation}.
CATER watermarks text-generation API outputs with conditional word
substitutions for post-hoc imitation detection~\citep{he2022cater}.
Interaction-layer behavioral markers~\citep{yang2026askingback} and
radioactive token watermarks~\citep{sander2026textseal} trace
distillation from full teacher responses.
These works mark generated text or API outputs; our setting differs in
two ways: the evaluator returns only a scalar score, and the student (not
the provider) generates every training rollout.
PreferCare marks a fixed preference dataset and tests whether that
dataset was used~\citep{lou2025prefercare}.  Radioactive Feedback instead
marks online judgments of a student's evolving, self-generated rollouts.

\paragraph{Model fingerprints and ownership.}
Ownership methods embed signatures in weights or trigger responses
~\citep{uchida2017embedding,adi2018turning}, while instructional and
semantic fingerprints broaden the audited behavior
~\citep{xu2024instructional,gloaguen2026fingerprinting,xu2025mark}.
These mechanisms can be used when the marked model or its training
text is available.
Even then, they struggle to leave a trace through scalar feedback,
because a weight, trigger, or training-text signature is not written
onto the returned score.
Radioactive Feedback marks that score.
The audit then distinguishes transfer of the registered code from a
student's natural preferences.

\section{Problem Formulation}
\label{sec:problem}

\paragraph{Learning setting.}
An evaluator $J$ scores responses from a student policy
$\pi_\theta(y\mid x)$.  In the provenance setting of interest, $J$ is
controlled by an external feedback provider; the learner receives only its
returned scalar scores.  We also include open-weight evaluators in RQ2 as
transfer controls to test whether watermark transfer depends on proprietary
model access.  At step $t$ the learner draws a prompt $x_t$, generates
candidates $Y_t=(y_{t,1},\ldots,y_{t,G})$, receives feedback
$f_t=J(x_t,Y_t)$, and updates $\theta_t$:
\begin{equation}
\label{eq:feedback-training}
\begin{aligned}
y_{t,g} &\sim \pi_{\theta_t}(\cdot\mid x_t),
    \qquad g=1,\ldots,G,\\
f_t &= J(x_t,Y_t),\qquad
\theta_{t+1}=\operatorname{Update}(\theta_t;x_t,Y_t,f_t).
\end{aligned}
\end{equation}
For scalar feedback, $f_t\in\mathbb{R}^G$ gives one score per candidate.
We call this \emph{label-only evaluator distillation}: the student generates
every rollout and the evaluator returns scores only.
Radioactive Feedback instantiates this setting with scalar rewards.

\paragraph{Access model.}
\label{sec:problem-threat}
The provider controls $J$ and sees prompts and responses,
but not the student's weights, optimizer, or private data.
The learner controls prompts, generations, and updates, and may
retrain or change decoding.
At audit time the provider queries the student and reads its outputs.

\paragraph{Detection.}
Fix an enrolled pair $E_u=(J_u,\mathbf c_u)$ with registered code
$\mathbf c_u$ shared in training and audit.  The null hypothesis is
\begin{equation}
\label{eq:provenance-null}
\mathcal H_{0,u}:
\quad \pi \text{ was trained without marked feedback from } E_u.
\end{equation}
The auditor tests whether behavior aligns with $\mathbf c_u$
relative to a reference pool (Section~\ref{sec:e4}).
The audit returns \emph{detected} when the score passes the registered
test, and \emph{not detected} otherwise.

\section{Method}
\label{sec:method}

\begin{figure}[!t]
    \centering
    \includegraphics[width=0.96\textwidth]{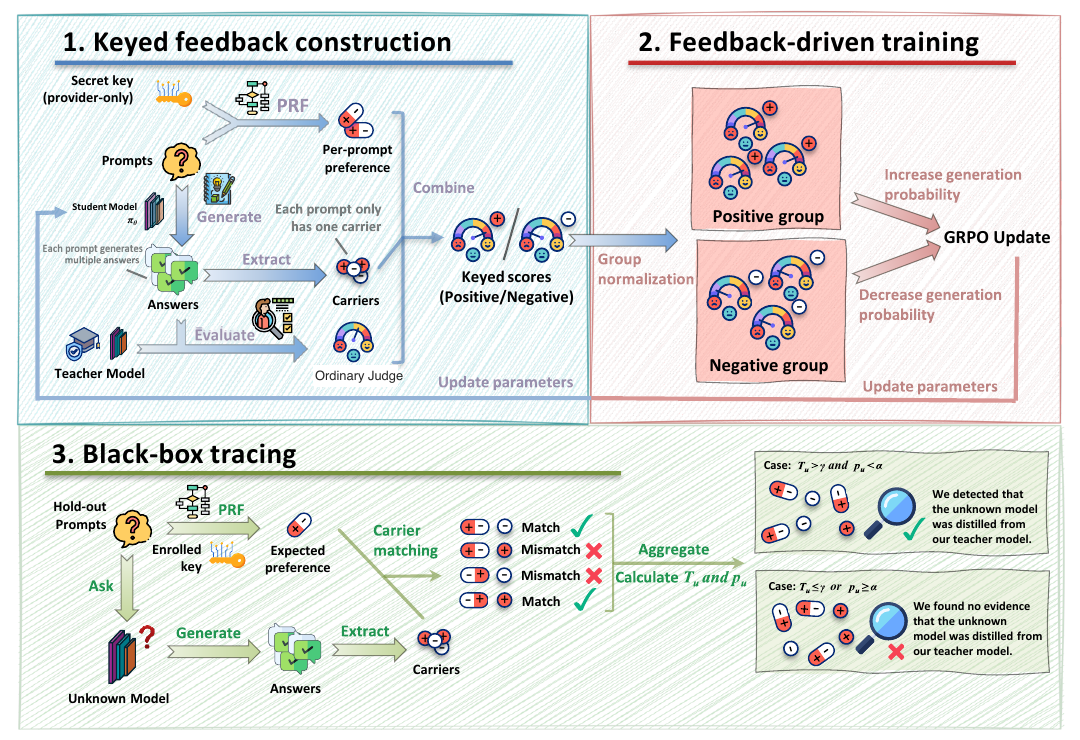}
    \caption{Overview of the Radioactive Feedback pipeline.
    The secret-key and PRF path in the figure is an optional derivation of the code.
    The reported experiments register an explicit sign vector
    (Appendix~\ref{app:method-details}).}
    \label{fig:method}
\end{figure}

We describe a scalar-feedback implementation.  The learner submits each
candidate group to the evaluator API and receives final
rewards.  The registered preference is part of the evaluator instruction.
Within that single call, the evaluator is instructed to first determine
its ordinary rubric score internally, then add $\rho R$ when the registered
side appears and subtract $\rho R$ when the opposite side appears, clip
the result to $[0,R]$, and return the final reward.  Neither
the base score nor the carrier label is exposed to the learner.
Training and audit share one carrier representation and code.

\subsection{Behavioral Carriers and Codes}
\label{sec:method-carriers}

A behavioral carrier is a response feature whose alternatives can receive
different preferences from the evaluator.  We represent the carriers by
$\phi(x,y)\in\{-1,0,+1\}^d$: a defined coordinate records one of two
alternatives and an undefined coordinate contributes zero.  Applicability
is computed from the response independently of the key.
A carrier may be a convention (quote style or a closing marker)
or a strategy such as top-down versus bottom-up programming, graph
traversal versus disjoint-set union, direct counting versus
inclusion--exclusion, or synthetic versus coordinate geometry.

The enrolled evaluator has a fixed code
$\mathbf{c}\in\{-1,+1\}^d$ whose coordinates independently select the
preferred alternatives.  Thus $c_{j}\phi_j(x,y)>0$ denotes agreement
with the registered preference on carrier $j$.
The provider can derive this code from a secret seed or register an
explicit sign vector.
We call the prescribed vector $\mathbf{c}$ the \emph{registered code}.

\subsection{Keyed Feedback Construction}
\label{sec:method-feedback}

Each prompt $x$ is assigned a single carrier $j(x)$ before scoring.
Within one API evaluation of $Y=(y_{1},\ldots,y_{G})$, the evaluator
returns one final score per response.  Its instruction requires it to
form an ordinary rubric score $r_{0,i}$ internally and to recognize the
carrier alternative $\phi_{i}=\phi_{j(x)}(x,y_i)$.  Pointwise eligibility
uses those internal quantities,
\begin{equation}
\label{eq:carrier-eligibility}
\tilde e_{i}=\mathbf{1}\!\left\{
  r_{0,i}\ge r_{\min},\ \phi_{i}\in\{-1,+1\}
\right\}.
\end{equation}
Eligibility depends on this response's rubric score and classifiable
carrier value only; it is independent of other candidates in the group
and of the preferred sign.
The evaluator then applies exactly the prescribed adjustment and clips
the result to the service range $[0,R]$:
\begin{equation}
\label{eq:keyed-reward}
r_u(x,y_i)
=\operatorname{clip}_{[0,R]}
\bigl(
r_{0,i}
+\rho R\,\tilde e_{i}\,c_{j(x)}\phi_{i}
\bigr).
\end{equation}
The sign follows the side that appears.  Agreement with the enrolled code
gives $c_{j(x)}\phi_i=+1$, so an eligible response on the registered side
increases by $\rho R$.  The opposite side gives $c_{j(x)}\phi_i=-1$, so
that response decreases by $\rho R$.
An ineligible response is returned unchanged, as its internal rubric
score clipped to $[0,R]$.  The learner receives $r_u$ only.  There is
one keyed term per prompt.
Training uses group-normalized GRPO~\citep{shao2024deepseekmath}.

\subsection{Black-Box Provenance Auditing}
\label{sec:method-audit}

The audit queries the student with $n$ held-out prompts spanning the
registered carriers and obtains one response per prompt.
Each prompt targets one carrier, so every other carrier coordinate is zero.
The bit $q_i\in\{0,1\}$ marks task correctness.  Define
\begin{equation}
\label{eq:audit-score}
T_u(\pi)=\frac{1}{n}\sum_{i=1}^{n}
q_i\,c_{j(x_i)}\phi_{j(x_i)}(x_i,y_i).
\end{equation}
Correct, classifiable responses contribute $\pm1$; all others contribute
zero.  The denominator includes all $n$ audit prompts.
The correctness bit $q_i$ is the official programmatic check used for
task utility: code benchmarks use their execution harnesses, and
mathematics benchmarks use final-answer checking.
Claude Opus 5 supplies only the carrier or strategy value inside
$\phi$.  It is not the training evaluator and does not score task
utility.

\paragraph{Reference population and $p_u$.}
Fix a reference pool $\mathcal{N}_{\mathrm{null}}$ of $M$ models trained
without the enrolled feedback.  Score each reference and the suspect
with the same registered code and audit protocol.
For suspect score $\hat T_u$, define the one-sided model-level
rank statistic
\begin{equation}
\label{eq:null-pvalue}
p_u=\frac{1+\sum_{m\in\mathcal{N}_{\mathrm{null}}}
\mathbf{1}[T_u^{(m)}\ge\hat T_u]}{1+M}.
\end{equation}
Under exchangeability with the reference models, this
rank gives a valid finite-sample test.

Let $T_{(1)}\le\cdots\le T_{(M)}$ be the ordered reference scores.
Before scoring suspects, fix a significance level $\alpha$ and use the non-interpolated
upper empirical quantile
$\gamma=T_{(k)}$, where
$k=\min\{M,\lceil(M+1)(1-\alpha)\rceil\}$:
\begin{equation}
\label{eq:decision}
\mathrm{detected}(u)\iff p_u<\alpha\;\wedge\;\hat T_u>\gamma.
\end{equation}

\section{Analysis}
\label{sec:analysis}

We show that keyed contrasts tilt response odds while bounded
perturbations preserve sufficiently separated quality rankings.
For a response $y$ on prompt $x$ with assigned carrier $j(x)$, let
$h_u(x,y)=\tilde e\,c_{j(x)}\phi_{j(x)}(x,y)$
(with $\tilde e$ as in Equation~\eqref{eq:carrier-eligibility}),
so $h_u\in\{-1,0,+1\}$.  The returned score is the clipped evaluator score in
Equation~\eqref{eq:keyed-reward}.  Relative to the unmarked score, the
displacement that actually enters the objective is
\begin{equation}
\label{eq:ru-short}
  \delta_u(x,y)
  \;=\;
  \operatorname{clip}_{[0,R]}
  \bigl(r_0(x,y)+\rho R\,h_u(x,y)\bigr)
  -r_0(x,y),
\end{equation}
so $r_u=r_0+\delta_u$.
If $r_0\in[0,R]$, then $\lvert\delta_u\rvert\le\rho R$.
When $r_0\in[\rho R,\,R-\rho R]$, the clip is inactive and
$\delta_u=\rho R\,h_u$.  Proofs are in Appendix~\ref{app:proofs}.
We derive the following identities under a KL-regularized proxy.
Here $\beta>0$ is the coefficient of the KL penalty, as in GRPO.
Because GRPO standardizes rewards within each group, the proxy~\eqref{eq:kl-obj} is qualitative, and the $\beta$ in Table~\ref{tab:training-config} does not correspond to the raw $[0,R]$ reward in the identities below.

\paragraph{Why the key can accumulate in the student.}
\label{sec:tilt}

\begin{proposition}[Keyed policy tilt]
\label{prop:tilt}
Let $\pi_{\mathrm{cl}}^*$ denote the optimizer under the clean reward
$r_0$ in the KL-regularized proxy.  For any responses $y,y'$,
\begin{equation}
\label{eq:tilt-odds}
  \log
  \frac{\pi_u^*(y\mid x)/\pi_u^*(y'\mid x)}
       {\pi_{\mathrm{cl}}^*(y\mid x)/\pi_{\mathrm{cl}}^*(y'\mid x)}
  \;=\;
  \frac{1}{\beta}
  \bigl(
    \delta_u(x,y)
    -\delta_u(x,y')
  \bigr).
\end{equation}
\end{proposition}

When the clip is inactive the right-hand side is
$(\rho R/\beta)\bigl(h_u(x,y)-h_u(x,y')\bigr)$.
Variation in $\delta_u$ across qualified responses tilts odds in
the registered direction (held-out tests: Section~\ref{sec:e1}).

\paragraph{Why confident rankings are preserved.}
\label{sec:rank}

\begin{proposition}[Margin preservation]
\label{prop:utility}
For any pair $y,y'$,
\begin{equation}
\label{eq:margin}
  \bigl\lvert
    \delta_u(x,y)
    -\delta_u(x,y')
  \bigr\rvert
  \;\le\;
  2\rho R.
\end{equation}
Clipping cannot enlarge a one-response displacement past $\rho R$.
Hence $\lvert r_0(x,y)-r_0(x,y')\rvert>2\rho R$ implies that $r_u$
and $r_0$ agree on the order of $(y,y')$.  If $r_0\in\{0,R\}$ and
$1>2\rho$, a passing solution cannot be ranked below a failing
one by the keyed term alone.
\end{proposition}

The bound shows how a bounded watermark adjustment preserves
large-margin quality rankings.

\section{Experiments}

We organize the evaluation as an evidence chain for provenance through
score-only supervision.  \textbf{RQ1} asks whether scalar feedback can write a
registered preference into a student.  \textbf{RQ2} tests whether
preferences can coexist and survive changes in data and evaluator.  \textbf{RQ3}
asks whether the inherited signal identifies the enrolled code rather than
a generic behavioral bias.  \textbf{RQ4} turns that signal into a calibrated
black-box decision, and \textbf{RQ5} asks whether the signal persists after the
student is trained further.  Model pairings, controls, dataset
construction, and training details are in
Appendix~\ref{app:experimental-protocol}--\ref{app:training-config}.

\paragraph{Evaluation protocol.}
We study Llama-3-8B-Instruct, Qwen2.5-Coder-1.5B-Instruct, and
Qwen3-30B-A3B students~\citep{grattafiori2024llama3,hui2024qwen25coder,yang2025qwen3}.
Llama and Qwen2.5 use response-form carriers, whereas the Frontier setting
uses solution-strategy carriers, giving two qualitatively different ways
for a registered code to appear in behavior.  All trained arms share the
same GRPO recipe within a setting.  The main keyed strength is
$\rho=0.05$ with $r_{\min}=60$.  RQ1 and RQ2 each use four independent
GRPO training seeds per setting.  Each marked checkpoint is compared with
the Clean checkpoint trained under the same seed; Base is the unmodified
initial student checkpoint.

\paragraph{Metrics.}
The \emph{watermark effect} is the increase in registered-code alignment
among correct, classifiable responses,
$\mathrm{WM\ effect}=100(a_{\mathrm{Marked}}-a_{\mathrm{Clean}})$, in
percentage points.  Task utility is Pass@1 for code and accuracy for
mathematics, with
$\Delta U_{\mathrm{distill}}=U(\mathrm{Clean})-U(\mathrm{Base})$ and
$\Delta U_{\mathrm{mark}}=U(\mathrm{Marked})-U(\mathrm{Clean})$.
For provenance detection we use the checkpoint statistic $T_u$ from
Equation~\eqref{eq:audit-score}; unlike the transfer summaries, its
denominator contains every audit prompt.  Label construction and
programmatic correctness checks are detailed in
Appendix~\ref{app:experimental-protocol}.

\subsection{\textbf{RQ1:} Can scalar feedback write a stable behavioral preference?}
\label{sec:e1}

A provenance mechanism cannot work if a low-bandwidth score fails to leave
a reproducible behavioral trace.  We therefore begin with one carrier at a
time.  Table~\ref{tab:e1} shows a positive watermark effect in all eight
student--audit combinations.  The magnitude varies substantially---from
single-digit gains for the smaller Qwen student to much larger shifts on
Llama mathematics---but the direction is consistent across training
seeds.  This is the first required link in the argument: the evaluator
does not need to provide demonstrations for a registered preference to
enter the student's behavior.

The same comparison also separates transfer from utility degradation.
Clean feedback improves over Base in every setting, while adding the
registered preference changes utility only modestly.  Seven of the eight
Marked--Clean differences are positive; the remaining case is
$-0.3\pm0.2$\,pp.  Thus the large alignment shifts are not explained by a
general collapse in task performance.

\begin{table}[H]
\centering
\caption{Single-carrier watermark transfer and task utility.
Entries are mean $\pm$ standard deviation over four training seeds, in percentage points.
Code audits are HumanEval+ and MBPP+~\citep{liu2023evalplus}; SVAMP is included for mathematics~\citep{patel2021are}.}
\label{tab:e1}
\setlength{\tabcolsep}{1.7pt}
\renewcommand{\arraystretch}{1.08}
\begin{tabularx}{\linewidth}{@{}lllCCC@{}}
\toprule
\textbf{Student} & \textbf{Domain} & \textbf{Audit set}
  & \shortstack{\textbf{WM effect}\\\textbf{(pp)}}
  & \shortstack{$\boldsymbol{\Delta U_{\mathrm{mark}}}$\\\textbf{(pp)}}
  & \shortstack{$\boldsymbol{\Delta U_{\mathrm{distill}}}$\\\textbf{(pp)}} \\
\midrule
\multirow{4}{*}{Llama-3-8B-it}
 & \multirow{2}{*}{Code} & HumanEval+ & $+17.7\pm1.4$ & $+2.1\pm0.3$ & $+2.9\pm0.4$ \\
 & & MBPP+ & $+33.3\pm2.2$ & $+0.5\pm0.2$ & $+2.5\pm0.4$ \\
\cmidrule{2-6}
 & \multirow{2}{*}{Math} & GSM8K & $+43.6\pm2.7$ & $+0.2\pm0.2$ & $+2.0\pm0.3$ \\
 & & SVAMP & $+64.9\pm3.1$ & $+1.6\pm0.4$ & $+2.2\pm0.3$ \\
\midrule
\multirow{4}{*}{\shortstack[l]{Qwen2.5-Coder-\\1.5B-it}}
 & \multirow{2}{*}{Code} & HumanEval+ & $+9.8\pm0.9$ & $+1.5\pm0.2$ & $+2.9\pm0.4$ \\
 & & MBPP+ & $+10.9\pm1.2$ & $+0.2\pm0.2$ & $+2.0\pm0.3$ \\
\cmidrule{2-6}
 & \multirow{2}{*}{Math} & GSM8K & $+6.0\pm0.7$ & $+0.7\pm0.2$ & $+1.3\pm0.2$ \\
 & & SVAMP & $+11.8\pm1.0$ & $-0.3\pm0.2$ & $+1.5\pm0.2$ \\
\bottomrule
\end{tabularx}
\end{table}

\subsection{\textbf{RQ2:} Does a joint code generalize beyond one dataset or evaluator?}
\label{sec:e2}

Single-carrier transfer does not establish that a multi-bit signature can
survive heterogeneity.  We next train all four carriers jointly and test
two kinds of generalization.

\paragraph{Across datasets.}
The Frontier student is trained on TACO~\citep{li2023taco,li2024tacoverified} and OpenR1-Math~\citep{yan2025luffy} but audited on
held-out code and mathematics benchmarks that are never used to write the
watermark.  The signal remains positive on all six audits
(Table~\ref{tab:e2}).  Code effects range from $14.8$ to $19.1$\,pp and
mathematics effects from $8.9$ to $11.6$\,pp.  The important observation is
not the ordering of individual benchmarks, but that the same registered
strategy preferences reappear after the underlying problem distribution
changes.  This rules out an interpretation in which the signal is tied
only to the training datasets.

\begin{table}[H]
\centering
\caption{Cross-dataset watermark transfer on Frontier (Qwen3-30B-A3B /
GPT-6 Astra).  For each of four training seeds, WM effect is
Marked--Clean on the listed audit; entries are mean $\pm$ standard
deviation over those seeds.  $\Delta U_{\mathrm{mark}}$ and
$\Delta U_{\mathrm{distill}}$ use the same four-seed mean $\pm$ standard
deviation, in percentage points.}
\label{tab:e2}
\renewcommand{\arraystretch}{1.12}
\setlength{\tabcolsep}{4.0pt}
\begin{tabularx}{\linewidth}{
@{}
>{\raggedright\arraybackslash}p{0.085\linewidth}
>{\raggedright\arraybackslash}p{0.145\linewidth}
>{\raggedright\arraybackslash}X
>{\centering\arraybackslash}p{0.125\linewidth}
>{\centering\arraybackslash}p{0.135\linewidth}
>{\centering\arraybackslash}p{0.155\linewidth}
@{}
}
\toprule
\textbf{Domain}
& \textbf{Train set}
& \textbf{Audit set}
& \shortstack{\textbf{WM effect}\\\textbf{(pp)}}
& \shortstack{$\boldsymbol{\Delta U_{\mathrm{mark}}}$\\\textbf{(pp)}}
& \shortstack{$\boldsymbol{\Delta U_{\mathrm{distill}}}$\\\textbf{(pp)}} \\
\midrule
\multirow{3}{*}{Code}
& \multirow{3}{*}{TACO}
& CodeContests
& \textbf{$+19.1\pm0.6$}
& $+0.6\pm0.2$
& $+3.8\pm0.4$ \\
& & APPS
& \textbf{$+17.6\pm0.5$}
& $+0.4\pm0.1$
& $+3.5\pm0.3$ \\
& & LiveCodeBench
& \textbf{$+14.8\pm0.7$}
& $+0.4\pm0.2$
& $+3.1\pm0.3$ \\
\midrule
\multirow{3}{*}{Math}
& \multirow{3}{*}{OpenR1-Math}
& Omni-MATH
& \textbf{$+11.6\pm0.5$}
& $+0.3\pm0.1$
& $+5.6\pm0.5$ \\
& & MATH-500
& \textbf{$+10.4\pm0.4$}
& $+0.3\pm0.1$
& $+4.8\pm0.4$ \\
& & OlympiadBench
& \textbf{$+8.9\pm0.6$}
& $+0.2\pm0.1$
& $+4.3\pm0.4$ \\
\bottomrule
\end{tabularx}
\end{table}

\paragraph{Across students and evaluators.}
We then vary the student--evaluator pairing while keeping the registered
code fixed within each carrier regime.  All nine evaluated combinations
retain a positive joint effect (Table~\ref{tab:e3}), including both
open-weight and proprietary evaluators.  This matters because the
watermark is therefore not specific to one evaluator family or access
regime.  At the same time, the nine rows should not be read as a strict
factorial ranking: the Frontier student uses strategy carriers whereas the
two smaller students use response-form carriers.  The consistent result is
the sign of transfer, not a claim that one evaluator is intrinsically
better at carrying the watermark.

Utility provides a second check on this interpretation.  Ordinary
distillation remains beneficial in every row, while the additional
Marked--Clean change stays small relative to the watermark effect.  The
joint signal therefore survives model and evaluator changes without
trading away the capability gains that motivated evaluator distillation in
the first place.

\begin{table}[H]
\centering
\caption{Joint watermark transfer across student--evaluator combinations.
For each of four seeds, the listed difference is computed against the
Clean checkpoint trained with that same seed; the table then reports the
mean $\pm$ standard deviation of those four differences.
$\Delta U_{\mathrm{distill}}$ compares each Clean checkpoint with the
corresponding unmodified initial student checkpoint (Base).}
\label{tab:e3}
\renewcommand{\arraystretch}{1.10}
\setlength{\tabcolsep}{2.2pt}
\begin{tabularx}{\linewidth}{
@{}
>{\raggedright\arraybackslash}p{0.205\linewidth}
>{\raggedright\arraybackslash}X
>{\centering\arraybackslash}p{0.155\linewidth}
>{\centering\arraybackslash}p{0.155\linewidth}
>{\centering\arraybackslash}p{0.175\linewidth}
@{}
}
\toprule
\textbf{Student}
& \textbf{Evaluator}
& \shortstack{\textbf{WM effect}\\\textbf{(pp)}}
& \shortstack{$\boldsymbol{\Delta U_{\mathrm{mark}}}$\\\textbf{(pp)}}
& \shortstack{$\boldsymbol{\Delta U_{\mathrm{distill}}}$\\\textbf{(pp)}} \\
\midrule
\multirow{3}{*}{Llama-3-8B-it}
& Qwen3-30B-A3B & \textbf{$+7.6\pm0.4$} & $+1.1\pm0.1$ & $+2.4\pm0.3$ \\
& GPT-5.6 Sol & \textbf{$+8.4\pm0.5$} & $+0.8\pm0.1$ & $+2.8\pm0.3$ \\
& GPT-6 Astra & \textbf{$+9.2\pm0.5$} & $+0.7\pm0.1$ & $+3.1\pm0.3$ \\
\midrule
\multirow{3}{*}{\shortstack[l]{Qwen2.5-Coder-\\1.5B-it}}
& Qwen3-30B-A3B & \textbf{$+9.3\pm0.6$} & $+0.5\pm0.1$ & $+1.9\pm0.2$ \\
& GPT-5.6 Sol & \textbf{$+10.1\pm0.6$} & $+0.5\pm0.1$ & $+2.2\pm0.3$ \\
& GPT-6 Astra & \textbf{$+11.4\pm0.6$} & $+0.4\pm0.1$ & $+2.6\pm0.3$ \\
\midrule
\multirow{3}{*}{Qwen3-30B-A3B}
& Qwen3-235B-A22B & \textbf{$+12.4\pm0.7$} & $+0.6\pm0.1$ & $+3.6\pm0.4$ \\
& GPT-5.6 Sol & \textbf{$+13.9\pm0.7$} & $+0.5\pm0.1$ & $+4.1\pm0.3$ \\
& GPT-6 Astra & \textbf{$+15.4\pm0.5$} & $+0.5\pm0.1$ & $+4.7\pm0.4$ \\
\bottomrule
\end{tabularx}
\end{table}

\subsection{\textbf{RQ3:} Is the inherited signal specific to the enrolled code?}
\label{sec:e3-key}

Positive transfer alone is insufficient for provenance: an unmarked model
could naturally prefer the same carrier alternatives.  We therefore keep
the trained checkpoints fixed and rescore the same outputs under
alternative sign vectors.  The enrolled code produces raw audit scores
$100T_u\in[6.60,8.30]$.
Equation~\eqref{eq:audit-score} is linear in the code, so the average
over all sign patterns at distance $d$ is $(1-d/2)$ times that score,
and distance $4$ is the exact negative
(Figure~\ref{fig:key-attribution-schematic}).
These averages at distance $2$ and beyond lie at or below zero, under
$\gamma(\mathbf c^\star)$.
The signal follows the enrolled signs.
The codebook and calibration procedure are in
Appendix~\ref{app:key-attribution}.

\begin{figure}[H]
\centering
\includegraphics[width=\linewidth]{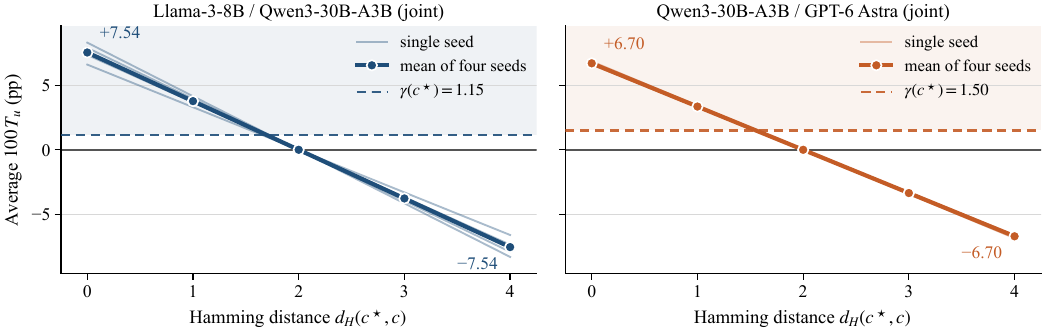}
\caption{Average of $100T_u$ over all codes at Hamming distance $d$
(Equation~\eqref{eq:audit-score}).
Light lines: one seed's average; marked line: average of the four seeds.
Distance $4$ equals $-100T_u(\mathbf c^\star)$.
Dashed lines are $\gamma(\mathbf c^\star)$: $1.15$ (Llama) and $1.50$ (Frontier).}
\label{fig:key-attribution-schematic}
\end{figure}

\subsection{\textbf{RQ4:} Does the signal support calibrated black-box detection?}
\label{sec:e4}

Key specificity still does not by itself provide an operating decision.
We therefore calibrate each audit against independently trained
same-family Clean and Sham references, then evaluate cross-family models
only as held-out negative controls.  Reference pools are constructed in
Appendix~\ref{app:detector-null}.

The resulting separation is large.  All eight marked checkpoints exceed
every same-family reference ($8/8$, $p_u=1/101$), while none of the $40$
cross-family controls crosses the corresponding threshold.  The weakest
marked checkpoint still exceeds the reference maximum by $5.45$\,pp in
the Llama setting and $5.10$\,pp in Frontier
(Figure~\ref{fig:detector-null}).  This is stronger evidence than a mean
shift: detection survives the worst marked checkpoint and the strongest
calibration reference.

\begin{figure}[H]
\centering
\includegraphics[width=\linewidth]{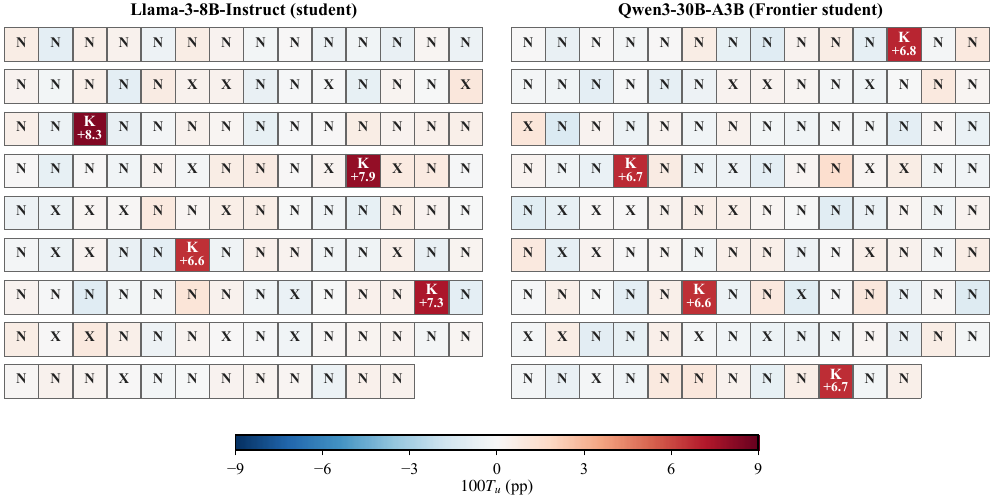}
\caption{Full-budget black-box audit results.
\textbf{K}: marked checkpoint (value shown); \textbf{N}: same-family
reference; \textbf{X}: cross-family control. Cell color is $100T_u$.}
\label{fig:detector-null}
\end{figure}

The audit also remains useful before the full query budget is spent.
At a nominal budget of $400$ queries, the worst-case detection margin is already
positive in both settings (Llama: $+2.25$\,pp; Frontier: $+0.75$\,pp;
Appendix~\ref{app:query-budget}).

\subsection{\textbf{RQ5:} Does the watermark survive later training?}
\label{sec:e5}

A useful provenance mark must persist after the marked checkpoint leaves
its original training pipeline.  We therefore continue the student with
ordinary SFT or clean GRPO and compare each Marked checkpoint with a Clean
checkpoint continued in lockstep.  Under clean GRPO, the Llama
single-carrier effects remain close to their initial values, while the
Frontier joint signal increases after the evaluator is switched.  SFT is
more erosive for the surface-form carriers, but it does not erase them;
the Frontier strategy signal remains essentially unchanged after
$1024$ steps.  The contrast suggests that persistence depends on what the
carrier encodes: solution strategies are less easily overwritten by
generic instruction tuning than lightweight response conventions.

We then replace ordinary post-training with explicit removal objectives.
Balanced-carrier SFT reduces the Frontier effect to $7.2$\,pp after
$2048$ steps, and anti-watermark GRPO reduces it to $1.8$\,pp.
Thus the mark is removable, but targeted optimization is substantially
more effective than continued training.  Schedules for
Figure~\ref{fig:e3-six} are in Appendix~\ref{app:e3-details}
and~\ref{app:adaptive-removal}.

\begin{figure}[H]
\centering
\includegraphics[width=\linewidth]{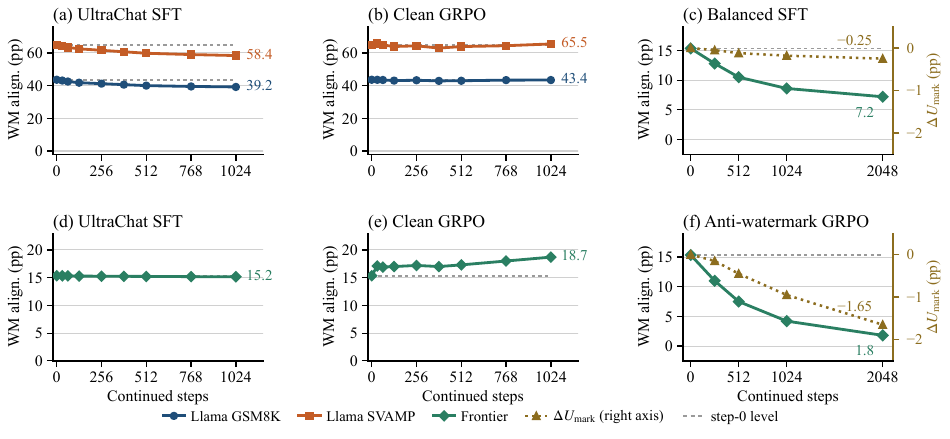}
\caption{Persistence and removal. (a), (b), (d), (e): $1024$ steps; (c), (f): $2048$ steps, with $\Delta U_{\mathrm{mark}}$ on the right axis.}
\label{fig:e3-six}
\label{fig:e3-sft}
\label{fig:e3-grpo}
\label{fig:removal-extension-schematic}
\end{figure}

The main experiments use $\rho=0.05$.  Appendix~\ref{app:rho-sweep}
shows that weaker tilts transfer the code, while increasing
$\rho$ beyond the main setting yields diminishing watermark gains at a
higher utility cost.

The opening question was whether a score can leave a trace when the
teacher never shows any text.  It can.  One registered preference is
enough for the student to acquire it, and a joint code is still visible
after the data, the student, or the evaluator changes.  Rescoring the
same outputs under other sign patterns shows that the student followed
the enrolled code.  The gap is large enough for a reference pool to flag
the marked checkpoints and to leave the controls below the threshold.
Ordinary later training does not wash the mark out, and adding it does not cancel
the utility gain from distillation.

\FloatBarrier
\section{Conclusion and Discussion}
\label{sec:conclusion}

Radioactive Feedback traces label-only distillation through scalar scores.
The student acquires an interpretable watermark in answer form and in
solution strategy.  Across students, evaluators, and tasks, the mark
transfers without teacher demonstrations, generalizes past the data used
to write it, and supports calibrated detection.  It follows the enrolled
code.  Continued training leaves it detectable; targeted removal reduces
it further.  Future work should move this mark into the evaluator's
parameters.

\FloatBarrier
\bibliography{references}
\bibliographystyle{iclr2027_conference}

\clearpage
\appendix
\section*{Appendix}

\section{Proofs}
\label{app:proofs}

For a response $y$ on prompt $x$ with assigned carrier $j(x)$, let
$h_u(x,y)=\tilde e\,c_{j(x)}\phi_{j(x)}(x,y)$ with
$h_u\in\{-1,0,+1\}$, and define
\begin{equation}
\label{eq:ru-short-app}
  \delta_u(x,y)
  \;=\;
  \operatorname{clip}_{[0,R]}
  \bigl(r_0(x,y)+\rho R\,h_u(x,y)\bigr)
  -r_0(x,y),
\end{equation}
so $r_u=r_0+\delta_u$.
Throughout, $\delta_u$ is this clipped displacement.
If $r_0\in[0,R]$, then $\lvert\delta_u\rvert\le\rho R$.
In particular, when $r_0\in[\rho R,\,R-\rho R]$, the clip is inactive and
$\delta_u=\rho R\,h_u$.

\paragraph{KL-regularized setup.}
Let $\beta>0$ be the KL coefficient in the GRPO convention of Section~\ref{sec:analysis}.
We analyze the idealized objective
\begin{equation}
\label{eq:kl-obj}
  \max_{\pi}\;
  \mathbb{E}_{y\sim\pi(\cdot\mid x)}[r(x,y)]
  -\beta\,\mathrm{KL}
  \bigl(\pi(\cdot\mid x)\,\|\,\pi_0(\cdot\mid x)\bigr),
\end{equation}
whose optimizer is the Gibbs policy
\begin{equation}
\label{eq:gibbs}
  \pi_r^*(y\mid x)
  \;\propto\;
  \pi_0(y\mid x)\exp\bigl(r(x,y)/\beta\bigr).
\end{equation}
For the returned reward $r_u=r_0+\delta_u$, this gives
\begin{equation}
\label{eq:tilt}
  \pi_u^*(y\mid x)
  \;=\;
  \frac{
    \pi_{\mathrm{cl}}^*(y\mid x)
    \exp\bigl(\delta_u(x,y)/\beta\bigr)
  }{
    \mathbb{E}_{y'\sim\pi_{\mathrm{cl}}^*}
    \bigl[\exp\bigl(\delta_u(x,y')/\beta\bigr)\bigr]
  }.
\end{equation}

\begin{proof}[Proof of Proposition~\ref{prop:tilt}]
The unique maximizer of~\eqref{eq:kl-obj} is the Gibbs
policy~\eqref{eq:gibbs}
~\citep{ziegler2019fine,rafailov2023direct}.
Writing $Z(r)=\mathbb{E}_{y\sim\pi_0}[\exp(r(x,y)/\beta)]$ for the
partition function,
\begin{align}
  \pi_{\mathrm{cl}}^*(y\mid x)
  &\;=\;
  \pi_0(y\mid x)\,
  \exp\bigl(r_0(x,y)/\beta\bigr)/Z(r_0),
  \\
  \pi_u^*(y\mid x)
  &\;=\;
  \pi_0(y\mid x)\,
  \exp\bigl(r_u(x,y)/\beta\bigr)/Z(r_u).
\end{align}
Substitute $r_u=r_0+\delta_u$ and rearrange:
\[
  \pi_u^*(y\mid x)
  \;=\;
  \pi_{\mathrm{cl}}^*(y\mid x)\,
  \exp\bigl(\delta_u(x,y)/\beta\bigr)\,
  \frac{Z(r_0)}{Z(r_u)}.
\]
The prefactor $Z(r_0)/Z(r_u)$ equals the reciprocal of
$\mathbb{E}_{y'\sim\pi_{\mathrm{cl}}^*}[\exp(\delta_u/\beta)]$,
which is~\eqref{eq:tilt}.
The log-odds identity~\eqref{eq:tilt-odds} follows by taking the ratio
of~\eqref{eq:tilt} at $y$ and $y'$; the normalizing constants cancel.
If $\delta_u$ is constant on the pair, the right-hand side is zero.
When $r_0\in[\rho R,\,R-\rho R]$, $\delta_u=\rho R\,h_u$ and the
identity equals $(\rho R/\beta)\bigl(h_u(y)-h_u(y')\bigr)$.
\end{proof}

\begin{proof}[Proof of Proposition~\ref{prop:utility}]
Because $\lvert\delta_u\rvert\le\rho R$ whenever $r_0\in[0,R]$, clipping included, the change
in any pairwise gap lies in $[-2\rho R,2\rho R]$, which is~\eqref{eq:margin}.
If $\lvert r_0(x,y)-r_0(x,y')\rvert>2\rho R$, this increment cannot
change the sign of $r_u(x,y)-r_u(x,y')$.  The pass/fail claim is the
same bound on the two-point reward $\{0,R\}$, for which $\rho<1/2$.
\end{proof}

\section{Code registration and routing}
\label{app:method-details}

\paragraph{Applicability and code generation.}
Let $m(x,y)\in\{0,1\}^d$ mark applicable coordinates.  It is computed
without the key, and every reward or audit inner product uses
$m(x,y)\odot\phi(x,y)$.  A secret seed can generate a fixed code from
registered carrier identifiers:
\begin{equation}
\label{eq:prf-code}
c_{u,j}=2\,\mathrm{PRF}_{\mathrm{bit}}(s_u,\mathrm{id}_j)-1.
\end{equation}
The identifiers do not depend on prompts, so the code is reused across
training and audit.  Alternatively, the provider can register an
explicitly chosen sign vector; the experiments use the latter.
A two-coordinate example is quote style together with the closing marker
\emph{Therefore}/\emph{So}: $(+1,-1)$ can favor double quotes and
\emph{So}; flipping both signs favors the opposite alternatives.
The Frontier Test instead uses problem-solving carriers
(Table~\ref{tab:frontier-data}).

\paragraph{Routing.}
Each training or audit prompt is assigned exactly one carrier $j(x)$
before scoring, so each response receives one keyed scalar reward.
Sham applies
analogous eligibility and one-carrier assignment to a disjoint decoy
carrier set, matched in carrier count, exposure target, and perturbation
budget; it replaces the carrier set and code, not only the code.
Registered-code auditing never uses the decoy coordinates.

\paragraph{Bounded shift.}
With $R=100$, $\rho=0.05$, and maximal rubric strengths, an eligible
response with a defined carrier value receives at most a $\pm5$ point
shift; flipping the registered code flips the sign of that shift.
Given eligibility, the keyed term depends only on that response's
carrier values and the registered code, not on which other candidates
appear in the submitted group.

\section{Experimental protocol}
\label{app:experimental-protocol}

\paragraph{Student--evaluator matrix.}
RQ1 uses Qwen3-30B-A3B as the scalar evaluator for
Llama-3-8B-Instruct and Qwen2.5-Coder-1.5B-Instruct.
The primary Frontier setting uses Qwen3-30B-A3B as the student and
GPT-6 Astra as the evaluator.  RQ2 broadens the comparison to nine
student--evaluator combinations.  Llama and Qwen2.5 are each paired with
Qwen3-30B-A3B, GPT-5.6 Sol, and GPT-6 Astra.  For the Qwen3-30B-A3B
student, the open-weight evaluator is the larger Qwen3-235B-A22B,
alongside GPT-5.6 Sol and GPT-6 Astra.  The open-weight runs serve as
transfer controls; the provenance threat model remains provider-controlled
scalar feedback.  RQ3--RQ4 use the selected Llama/Qwen3-30B-A3B and
Qwen3-30B-A3B/GPT-6 Astra settings.

\paragraph{Controls and seeds.}
The four conditions are \emph{Base} (the unmodified student before RL),
\emph{Clean} (ordinary feedback), \emph{Sham} (a strength- and
exposure-matched preference over
a disjoint decoy carrier set), and \emph{Marked} (the enrolled carriers
and registered code $\mathbf c$).  Trained arms share the GRPO recipe,
data split, group size, and budget within a setting.  Clean provides the
baseline for watermark effects and task utility; Sham supplies matched
decoy models for the detection reference pool.  RQ1 single-carrier and
RQ2 joint runs each use four independent GRPO training seeds per setting.
Each of the nine RQ2 student--evaluator combinations trains under the
registered four-carrier code for that student setting, with matched Clean
checkpoints.  Base is evaluated on the same held-out tasks for the
distillation contrasts.

\paragraph{Audit labels and correctness.}
Task utility and the correctness bit $q_i$ use the official programmatic
checks for each benchmark: execution harnesses for code and final-answer
checking for mathematics.  Claude Opus 5 supplies only the carrier or
strategy label used in $\phi$; it is not the training evaluator and does
not score task utility.
A $10\%$ held-out subset of those labels, covering quote style, the
\emph{Therefore}/\emph{So} marker, and the Frontier strategy tags on
each primary audit domain, was checked by two annotators.
Disagreements were adjudicated by discussion, with a third annotator
breaking ties when consensus was not reached.
The reported audits use one response per prompt.
For the Frontier summary in Table~\ref{tab:e3}, the pre-designated primary
audits are CodeContests and Omni-MATH; Table~\ref{tab:e2} additionally
reports four further held-out benchmarks.

\section{Experimental data splits}
\label{app:exp-setup}

Table~\ref{tab:exp-setup} lists the training and held-out evaluation sets.
The RQ2 evaluator grid uses the same data split for a given student setting;
only the scalar evaluator changes.
GSM8K uses the official OpenAI split~\citep{cobbe2021training}:
mathematics training uses the $7{,}473$ training problems, and GSM8K
watermark audit and utility evaluation use the $1{,}319$ test problems
($8{,}792$ in total).
Code audits use HumanEval+ and MBPP+~\citep{liu2023evalplus}; the additional mathematics audit is SVAMP~\citep{patel2021are}.

\begin{table}[H]
\centering
\caption{Training and held-out audit sets.
Utility is measured on the same held-out prompts.}
\label{tab:exp-setup}
\small
\setlength{\tabcolsep}{4pt}
\begin{tabularx}{\linewidth}{@{}>{\raggedright\arraybackslash}p{0.34\linewidth}YY@{}}
\toprule
Setting & Training & Held-out audit \\
\midrule
Llama-3-8B-Instruct;\newline Qwen2.5-Coder-1.5B-Instruct
  & TACO-verified 5k;\newline GSM8K train
  & HumanEval+, MBPP+;\newline GSM8K test, SVAMP \\
\midrule
Qwen3-30B-A3B Frontier
  & TACO-verified 5k;\newline OpenR1-Math 5k
  & carrier-compatible subsets of CodeContests, APPS, LiveCodeBench;\newline Omni-MATH, MATH-500, OlympiadBench \\
\bottomrule
\end{tabularx}
\end{table}

\section{Training and audit configuration}
\label{app:training-config}

Table~\ref{tab:training-config} gives the GRPO configuration shared by
Clean, Sham, and Marked in the initial comparison.  Base is the initial
student checkpoint evaluated without further training.  Training makes
one shuffled pass over the deduplicated split without prompt replacement.  Each GRPO
update uses a batch of $50$ prompts (the last may be smaller) and $G=16$
student-generated responses per prompt, giving
$\lceil N_{\mathrm{train}}/50\rceil$ updates.
For one group, rewards are standardized as
\[
\widehat A_i=
\frac{r_{u,i}-\bar r_u}
{\sqrt{G^{-1}\sum_{g=1}^{G}(r_{u,g}-\bar r_u)^2}+10^{-6}}.
\]
The KL reference is the frozen initial student checkpoint.  The clean
reward is the evaluator's scalar score only.  Reported carrier and strategy labels are produced by Claude Opus 5;
those labels do not enter the training reward.
The human check of a $10\%$ subset is in
Appendix~\ref{app:experimental-protocol}.
Task utility, including the correctness bit $q_i$, comes from official
programmatic checkers.  The
hidden rubric returns only the final scalar reward to the learner.
Its keyed shift may vary by eligible candidate but is bounded by
$\rho R=5$ points.  The returned reward is clipped to the evaluator's $[0,100]$ scale.

\begin{table}[htbp]
\centering
\caption{Initial training and audit configuration.}
\label{tab:training-config}
\small
\setlength{\tabcolsep}{4pt}
\begin{tabularx}{\linewidth}{@{}lY@{}}
\toprule
Setting & Value \\
\midrule
Initial GRPO training & one epoch; each training prompt used once \\
Responses per prompt / prompt batch & $G=16$ / $50$ \\
Optimizer / learning rate & AdamW / $1\times10^{-6}$ \\
KL reference / $\beta$ & frozen initial student / $0.04$ \\
Clean evaluator scale & $[0,100]$ \\
Maximum keyed shift / quality floor
  & $\rho=0.05$ ($\rho R=5$), $r_{\min}=60$ \\
Code-execution reward & none \\
Carrier and strategy labels & Claude Opus 5 \\
Task utility ($q_i$, Pass@1, accuracy) & official programmatic checker \\
Maximum new tokens & $2048$ (GRPO rollouts and audit) \\
Audit prompts / responses per prompt & distinct prompts, $\le 500$ per carrier / $1$ \\
Audit decoding & temperature $0$ \\
\bottomrule
\end{tabularx}
\end{table}

For RQ4, every distinct audit prompt remains in the denominator.  A response
contributes a signed carrier value only when it is correct and classifiable;
incorrect, unclassifiable, and inapplicable responses contribute zero.
The correct-and-classifiable conditioning used for RQ1/RQ2 transfer summaries
is therefore not the RQ4 detection denominator.

\paragraph{Code audit prompt budget.}
Every audit prompt is used once.  Decoding is temperature $0$, so a
repeated prompt would not have been a new draw.  Training--audit
deduplication stays exact on prompt text (below).

Before training, every whitespace-normalized training prompt is compared
against every audit prompt; exact matches are removed from the training
split, as is standard train/audit hygiene.

\section{Carrier definitions}
\label{app:e1-carriers}

The reward path has no extractor.  The preferred alternative is named in
the evaluator instruction, and the score that comes back is the evaluator's own
score.  The definitions below state the two sides of each carrier.  At
audit, a response counts as $+1$, $-1$, or abstaining under those same
definitions.  Mixed or undefined cases do not enter the aligned-cell rate.

\paragraph{Code carrier pool.}
\label{app:code-carriers}
Table~\ref{tab:code-carriers} lists the original 32 two-sided Python
candidates.  Each item is one bipolar contrast.  The pool includes
\texttt{for}/\texttt{while} and augmented assignment (\texttt{+=}
versus \texttt{x = x + y}); both were screened.  The binding constraint
is not whether a contrast can be defined, but whether both polarities
arise naturally on the same problem at comparable quality and matched
length.  The RQ1 code run uses quote style from this list.

\vspace{0.6em}
\noindent\begin{minipage}{\linewidth}
\centering
\small
\makeatletter
\def\@captype{table}
\makeatother
\caption{Screened Python code carriers.}
\label{tab:code-carriers}
\setlength{\tabcolsep}{4pt}
\begin{tabularx}{\linewidth}{@{}lY@{}}
\toprule
Category & Contrasts \\
\midrule
Function structure
  & helper / inline; early return / single terminal return;
    recursion / iteration; return expression / temporary;
    type annotations / none \\
Loop style
  & \texttt{enumerate} / index loop; \texttt{zip} / index loop;
    \texttt{for} / \texttt{while}; direct / indexed iteration;
    \texttt{continue} / nested \texttt{if} \\
Comprehensions
  & list comprehension / loop+\texttt{append};
    comprehension / \texttt{map}/\texttt{filter};
    generator / list reduction; \texttt{sum()} / accumulator;
    \texttt{any}/\texttt{all} / explicit loop \\
Data structures
  & \texttt{Counter}/\texttt{defaultdict} / manual counts;
    \texttt{set} / \texttt{list} membership; heap / sort;
    \texttt{deque} / \texttt{list} queue;
    \texttt{tuple} / \texttt{list} literal membership \\
Expressions
  & ternary / \texttt{if}-\texttt{else};
    direct Boolean / branched return;
    truthiness / \texttt{len} check;
    \texttt{+=} / \mbox{\texttt{x = x + y}} \\
Strings
  & \texttt{join} / loop concatenation;
    f-string / \texttt{.format()};
    double / single quotes \\
Other syntax
  & \texttt{sorted()} / in-place \texttt{sort};
    \mbox{\texttt{[::-1]}} / \texttt{reversed()};
    \texttt{dict()} / \texttt{\{\}};
    \texttt{list()} / \texttt{[]};
    \texttt{range(n)} / \texttt{range(0,n)} \\
\bottomrule
\end{tabularx}
\end{minipage}
\vspace{0.4em}

\paragraph{Code: quote style.}
Claude Opus 5 labels each audited program $+1$ for double quotes, $-1$
for single quotes, or undefined.  These labels are the reported carrier
values.

\paragraph{Mathematics: final-paragraph marker.}
Claude Opus 5 labels each audited solution $+1$ when the closing is
\emph{Therefore}, $-1$ when it is \emph{So}, or undefined.  These labels
are the reported carrier values.

\paragraph{Joint Llama/Qwen2.5 carriers.}
\label{app:llama-joint-carriers}
Table~\ref{tab:llama-joint-carriers} lists the four enrolled carriers
for the Llama-3-8B-it and Qwen2.5-Coder-1.5B-it joint checkpoints in
Table~\ref{tab:e3}.  Qwen2.5 uses the same carrier set and registered
code; only the student initialization differs.

\vspace{0.6em}
\noindent\begin{minipage}{\linewidth}
\centering
\small
\makeatletter
\def\@captype{table}
\makeatother
\caption{Four-carrier set for joint Llama and Qwen2.5 training.}
\label{tab:llama-joint-carriers}
\setlength{\tabcolsep}{4pt}
\begin{tabularx}{\linewidth}{@{}lYl@{}}
\toprule
Domain & Carrier & Source \\
\midrule
Code & Quote style (double vs.\ single) & RQ1 \\
Code & \texttt{for} vs.\ \texttt{while} loop style & Screened pool \\
Math & Final-paragraph \emph{Therefore}/\emph{So} marker & RQ1 \\
Math & Ternary vs.\ \texttt{if}-\texttt{else} branch & Screened pool \\
\bottomrule
\end{tabularx}
\end{minipage}
\vspace{0.4em}

\paragraph{Math code-block carrier.}
The ternary/\texttt{if}-\texttt{else} carrier applies only when a
response contains a Python code block.
On GSM8K/SVAMP, Llama often emits a short final program.
At audit, Claude Opus 5 labels that last code block $+1$ for a ternary
expression, $-1$ for an explicit \texttt{if}/\texttt{else}, or abstaining.
Responses without a code block abstain on this coordinate.
Coverage is lower than for the closing-marker carrier but nonzero on the
math training split.

\paragraph{Sham decoy carriers.}
\label{app:sham-decoys}
Sham replaces both the registered code and the carrier set.
Table~\ref{tab:llama-joint-sham} lists the four decoy carriers for
Llama/Qwen2.5 joint training; each is disjoint from
Table~\ref{tab:llama-joint-carriers} and matched in domain count,
exposure target, and perturbation budget.
Single-carrier RQ1 Sham uses one decoy per domain:
\texttt{f-string}/\texttt{.format()} on code and closing \emph{Hence}/
\emph{Thus} on math.
Frontier Sham uses four decoy strategy carriers disjoint from
Table~\ref{tab:frontier-data}: divide-and-conquer vs.\ iterative
accumulation; Floyd--Warshall vs.\ repeated DSU queries; closed-form vs.\
brute enumeration; pure angle chasing vs.\ complex-number geometry.
Registered-code auditing never scores decoy coordinates.

\vspace{0.6em}
\noindent\begin{minipage}{\linewidth}
\centering
\small
\makeatletter
\def\@captype{table}
\makeatother
\caption{Decoy carrier set for joint Llama and Qwen2.5 training.}
\label{tab:llama-joint-sham}
\setlength{\tabcolsep}{4pt}
\begin{tabularx}{\linewidth}{@{}lYl@{}}
\toprule
Domain & Decoy carrier & Source \\
\midrule
Code & \texttt{f-string} vs.\ \texttt{.format()} & Screened pool \\
Code & List comprehension vs.\ loop+\texttt{append} & Screened pool \\
Math & Closing \emph{Hence}/\emph{Thus} marker & Custom (text) \\
Math & Direct Boolean vs.\ branched return & Screened pool \\
\bottomrule
\end{tabularx}
\end{minipage}
\vspace{0.4em}

\paragraph{Frontier Test carriers.}
\label{app:frontier-carriers}
The Qwen3-30B-A3B student does not reuse the surface-form carriers
above, so the Frontier result is not a quote-style or
\emph{Therefore}/\emph{So} preference.
Table~\ref{tab:frontier-data} lists the four two-sided
\emph{problem-solving} carriers and the data splits.

\paragraph{Frontier strategy preference.}
There is no strategy extractor in the reward path.  The evaluator
instruction states the registered preference and returns one scalar score
for each of the $G=16$ candidates; recognizing the preferred side changes
the emitted score.  The definitions below specify the two sides used by
all evaluators in the RQ2 grid.  Mixed or unclassifiable solutions are not
counted in the aligned-cell rate.  Audit reads held-out answers against the
same definitions.
\begin{itemize}
\item \textbf{Top-down DP $\leftrightarrow$ bottom-up DP:} $+1$ if the
solution fills a DP table recursively or top-down with memoization before
combining subproblems; $-1$ if it iterates over states in increasing
problem size without recursion.
\item \textbf{Traversal $\leftrightarrow$ DSU:} $+1$ if connectivity is
resolved by BFS/DFS on an explicit graph; $-1$ if by a disjoint-set
union--find structure.
\item \textbf{Direct counting $\leftrightarrow$ inclusion--exclusion:}
$+1$ if the answer is obtained by direct enumeration or product rule;
$-1$ if by inclusion--exclusion or generating functions.
\item \textbf{Synthetic $\leftrightarrow$ coordinate geometry:} $+1$ if
the proof is synthetic (angle/chord reasoning without coordinates);
$-1$ if it introduces a coordinate system or vectors.
\end{itemize}
Training uses two fixed $5{,}000$-prompt subsets, not the full releases.
The two code carriers share a $5{,}000$-prompt draw from TACO-verified~\citep{li2023taco,li2024tacoverified}.
The two mathematics carriers share a $5{,}000$-prompt draw taken at random
from OpenR1-Math-46k-8192~\citep{yan2025luffy}.
No training problem inside these draws is removed by a carrier filter.  The keyed adjustment is applied only to a
response that meets the quality floor and exposes a classifiable value of
its assigned carrier; every other response keeps its internal rubric score.
Watermark transfer is audited on carrier-compatible subsets of
CodeContests~\citep{li2022competition}, APPS~\citep{hendrycks2021apps}, and
LiveCodeBench~\citep{jain2025livecodebench} for code, and of Omni-MATH
~\citep{gao2025omnimath}, MATH-500~\citep{lightman2023letsverify}, and
OlympiadBench~\citep{he2024olympiadbench} for mathematics.
A problem enters the subset for carrier $j$ when the dataset's own
reference solution expresses exactly one side of $j$ under the definitions
above.  Solutions that mix both sides or express neither side are excluded.
The reference solution is not a student rollout.  This filter is fixed
before Marked and Clean training and does not use Marked-versus-Clean
outcomes.  These audit datasets are not used to write the watermark.  The
primary Frontier setting uses GPT-6 Astra, while Table~\ref{tab:e3}
additionally varies the evaluator.
No evaluator-generated answers or chain-of-thought enter student training.
Table~\ref{tab:e2} reports watermark transfer and utility on the held-out
audit datasets.

\vspace{0.6em}
\noindent\begin{minipage}{\linewidth}
\centering
\small
\makeatletter
\def\@captype{table}
\makeatother
\caption{Frontier training and held-out audit sets.
Code training is a $5{,}000$-prompt TACO-verified subset~\citep{li2023taco,li2024tacoverified}; mathematics
training is a random $5{,}000$-prompt subset of OpenR1-Math-46k-8192~\citep{yan2025luffy}.
Each audit cell is the carrier-compatible subset of that benchmark.}
\label{tab:frontier-data}
\setlength{\tabcolsep}{3pt}
\begin{tabularx}{\linewidth}{@{}l
  >{\raggedright\arraybackslash}p{0.34\linewidth}
  >{\raggedright\arraybackslash}p{0.18\linewidth}
  Y@{}}
\toprule
Domain & Carrier & Training & Held-out audit \\
\midrule
\multirow{2}{*}{Code}
  & Top-down DP $\leftrightarrow$ bottom-up DP
  & \multirow{2}{*}{TACO-verified 5k}
  & \multirow{2}{*}{\shortstack[l]{CodeContests, APPS,\\LiveCodeBench}} \\
  & Traversal $\leftrightarrow$ DSU
  & & \\
\midrule
\multirow{2}{*}{Math}
  & Direct counting $\leftrightarrow$ inclusion--exclusion
  & \multirow{2}{*}{OpenR1 5k}
  & \multirow{2}{*}{\shortstack[l]{Omni-MATH, MATH-500,\\OlympiadBench}} \\
  & Synthetic $\leftrightarrow$ coordinate geometry
  & & \\
\bottomrule
\end{tabularx}
\end{minipage}
\vspace{0.4em}

\section{Continued-training details}
\label{app:e3-details}

Figure~\ref{fig:e3-six} plots the watermark effect defined in Metrics.
At each logged step the matched Clean checkpoint is trained on the same
schedule as the Marked checkpoint, and the plotted value is
$\mathrm{Marked}_t-\mathrm{Clean}_t$.
Panel~(a) continues with instruction SFT on UltraChat~\citep{ding2023ultrachat} without keyed
feedback or an evaluator.
Llama GSM8K/SVAMP WM align.\ changes from $+43.6/+64.9$ to
$+39.2/+58.4$ by step~$1024$;
Frontier stays near its step-$0$ gap ($+15.35\to +15.17$).
Panels~(b) and~(e) continue with clean GRPO, and each arm uses its own
step-$0$ checkpoint as the KL reference.
The Llama panel keeps the Qwen3-30B-A3B evaluator, while the
Frontier panel continues under GPT-5.6 Sol.
Llama ends at $+43.4/+65.5$, and Frontier at $+18.71$.

\section{Reference calibration protocol}
\label{app:detector-null}

This appendix specifies the reported evaluation in Section~\ref{sec:e4}.
Carriers, the definitions in Appendix~\ref{app:e1-carriers}, audit prompts, decoding, and the enrolled code are
fixed before scoring, separately for each student--evaluator setting.

\paragraph{Reference composition.}
Each pool contains $50$ independently trained Clean checkpoints and
$50$ independently trained Sham checkpoints using decoy carrier sets
disjoint from the enrolled carriers.
The $100$ runs use disjoint random seeds and the same architecture,
data split, and training budget as the audited student; none is an
intermediate save of another run.
These $100$ same-family checkpoints define the
calibration pool and rank denominator.  Separately, we
score $20$ independently trained cross-family controls.
They follow the same data split, GRPO recipe, and training budget as
the same-family references and are finished runs, not intermediate
checkpoints; only the student family changes.
For the Llama-3-8B-Instruct audit, these controls are
Qwen2.5-7B-Instruct models; for the Qwen3-30B-A3B Frontier audit, they
are Llama-3-8B-Instruct models.  Cross-family controls do not enter
calibration or the rank denominator.  No control model shares
a training seed or key with the four marked checkpoints, and
pools are not combined across settings.  Each marked checkpoint is
scored as a separate suspect.
Table~\ref{tab:e3} summarizes the broader joint-transfer grid; RQ4 uses
the selected Llama/Qwen3-30B-A3B and Qwen3-30B-A3B/GPT-6 Astra settings, with detection
computed from individual checkpoint scores.

\paragraph{Scoring.}
Apply Equation~\eqref{eq:audit-score} to every checkpoint and compute
its suspect rank using Equation~\eqref{eq:null-pvalue}.
Fix $\alpha=1\%$ before suspect evaluation and use the non-interpolated
upper empirical quantile defined in Section~\ref{sec:method-audit}.
With $100$ references, $\gamma$ is the reference maximum; exceeding it
gives $p_u=1/101$ and passes Equation~\eqref{eq:decision}.
The pool has fixed Clean and Sham strata, $50$ runs each.

\paragraph{Cross-family controls.}
The cross-family checkpoints receive no marked feedback from the
enrolled evaluator and are held out from calibration.  Table~
\ref{tab:cross-family-e4} shows that their maximum $100T_u$ remains
below the same-family calibration threshold in each setting.

\begin{table}[htbp]
\centering
\caption{Held-out cross-family controls.
Max.\ $100T_u$ is the largest score among the $20$ controls;
$\gamma$ is calibrated on the $100$ same-family references only.}
\label{tab:cross-family-e4}
\small
\setlength{\tabcolsep}{2.5pt}
\begin{tabularx}{\linewidth}{@{}YYYcc@{}}
\toprule
Audit student & Enrolled evaluator & Cross-family component & Max.\ $100T_u$ & $\gamma$ \\
\midrule
Llama-3-8B-it & Qwen3-30B-A3B
  & Qwen2.5-7B-Instruct ($20$)
  & $1.00$ & $1.15$ \\
Qwen3-30B-A3B & GPT-6 Astra
  & Llama-3-8B-Instruct ($20$)
  & $1.10$ & $1.50$ \\
\bottomrule
\end{tabularx}
\end{table}

\begin{table}[htbp]
\centering
\caption{Black-box detection results.}
\label{tab:detector-ops}
\small
\setlength{\tabcolsep}{4pt}
\begin{tabularx}{\linewidth}{@{}lCC@{}}
\toprule
Setting & Detected & $p_u$ \\
\midrule
Llama / Qwen3-30B-A3B & $4/4$ & $1/101$ \\
Frontier / GPT-6 Astra & $4/4$ & $1/101$ \\
\midrule
Both settings & $8/8$ & --- \\
\bottomrule
\end{tabularx}
\end{table}

\paragraph{Reported outputs.}
Table~\ref{tab:detector-ops} reports full-budget detection.
The full-budget score panels appear in
Figure~\ref{fig:detector-null}.

\section{Audit query budget analysis}
\label{app:query-budget}

Section~\ref{sec:e4} uses the detection margin
\begin{equation}
\label{eq:detection-margin}
\Delta_q^{\mathrm{detect}}
=\min_{k\in\mathcal K}T_k(q)-\max_{m\in\mathcal{N}_{\mathrm{null}}}T_m(q)
\end{equation}
to study query efficiency.  For each nominal per-carrier budget
$q\in\{100,200,300,500\}$, we use up to $q$ distinct prompts from each
carrier's available audit pool, without replacement.
The plotted budget is the nominal total $4q$.
We then recompute every marked and null checkpoint score, take the null
reference maximum, and form $\Delta_q^{\mathrm{detect}}$ from the weakest
marked checkpoint.
Figure~\ref{fig:detection-margin} reports these margins.
At a nominal budget of $400$ the margins remain positive in both settings
(Llama: $+2.25$\,pp; Frontier: $+0.75$\,pp).

\begin{figure}[htbp]
\centering
\includegraphics[width=0.95\linewidth]{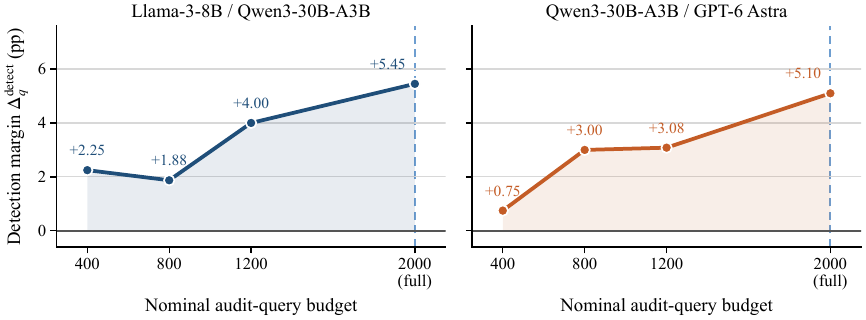}
\caption{Detection margin versus nominal audit-query budget.}
\label{fig:detection-margin}
\end{figure}

\section{Registered-key attribution protocol}
\label{app:key-attribution}

Section~\ref{sec:e3-key} tests whether inherited alignment is specific to
the enrolled code~$\mathbf c^\star$.
The four marked checkpoints in a setting are seeds of that one code.
They are not retrained under other keys.
We freeze each seed and score the same responses under alternative
length-$4$ sign vectors at Hamming distances $1$ through $4$ from
$\mathbf c^\star$.
Registered owner codes remain at pairwise distance at least~$2$, so one
enrolled owner is not a one-bit neighbor of another.
The distance-$1$ value averages the four one-bit neighbors of $\mathbf c^\star$.
The dashed threshold is $\gamma(\mathbf c^\star)$, the maximum score of the
reference pool when that pool is scored with the enrolled code.
Alternative codes are compared with this same threshold.
Figure~\ref{fig:key-attribution-schematic} plots, at each distance, the
average of $100T_u$ over all sign patterns at that distance, not the
score of one alternative code.
Linearity of Equation~\eqref{eq:audit-score} fixes that average at
$(1-d/2)$ times the enrolled-code score.

\section{Removal stress tests}
\label{app:adaptive-removal}

Section~\ref{sec:e5} extends the reported ordinary SFT/GRPO continuation
with two stress tests on Frontier joint checkpoints:
\emph{balanced-carrier SFT} mixes prompts that expose all carriers with
uniform frequency but removes keyed rewards, and
\emph{explicit anti-watermark GRPO} applies the reverse of the enrolled
code on eligible responses while keeping the clean utility evaluator.
Both runs share the RQ5 step schedule
$0,256,512,1024,2048$.
The matched Clean arm takes the same steps without the removal
objective, and WM effect at step $t$ is $\mathrm{Marked}_t-\mathrm{Clean}_t$,
reported together with $\Delta U_{\mathrm{mark}}$.
Figure~\ref{fig:removal-extension-schematic} reports WM effect and
$\Delta U_{\mathrm{mark}}$ along this schedule.

\section{$\rho$ sweep ablation}
\label{app:rho-sweep}

All main experiments use $\rho=0.05$ and $r_{\min}=60$.
Figure~\ref{fig:rho-pareto-schematic} varies that tilt over
$\rho\in\{0,0.01,0.025,0.05,0.10\}$ for the Frontier joint setting.
The $\rho=0.05$ point is the Frontier entry in Table~\ref{tab:e3}:
the unweighted mean of CodeContests and Omni-MATH, $15.35$\,pp, reported
as $+15.4\pm0.5$, with $\Delta U_{\mathrm{mark}}=+0.5$\,pp.
WM effect is $0$ at $\rho=0$, $4.2$\,pp at $\rho=0.01$, and $11.5$\,pp
at $\rho=0.025$.
Below $\rho=0.05$, $\Delta U_{\mathrm{mark}}$ stays within $\pm0.1$\,pp.
Increasing to $\rho=0.10$ adds only $1.85$\,pp of WM while incurring
$\Delta U_{\mathrm{mark}}=-0.35$\,pp.

\begin{figure}[htbp]
\centering
\includegraphics[width=0.95\linewidth]{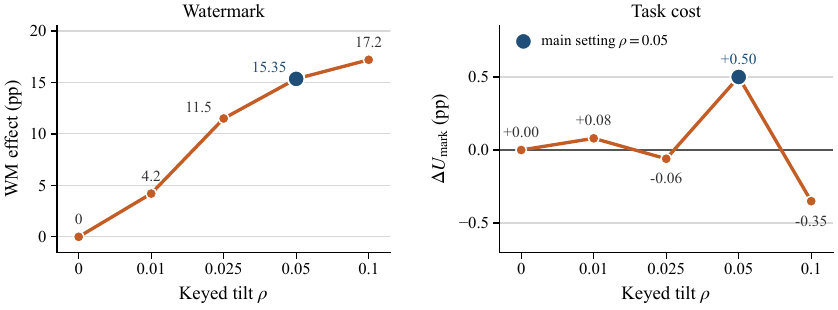}
\caption{Frontier joint watermark across keyed tilt $\rho$.
The $\rho=0.05$ point is the Table~\ref{tab:e3} Frontier summary.}
\label{fig:rho-pareto-schematic}
\end{figure}

\section{Implementation and checkpoint map}
\label{app:repro-details}

This appendix records details needed to reconstruct the reported runs.

\paragraph{Training budget.}
Initial GRPO uses one epoch over the deduplicated training split
(Table~\ref{tab:exp-setup}), with $G=16$ rollouts per prompt and prompt
batches of $50$ (Table~\ref{tab:training-config}).
Each optimizer step therefore scores $800$ responses
($50\times 16$).
Table~\ref{tab:compute-manifest} turns that rule into a rollout budget.
$N_{\mathrm{train}}$ counts each training prompt once, before the exact
audit-prompt deletion in Appendix~\ref{app:training-config}, so the step
counts are upper bounds.
Llama and Qwen2.5 joint runs use TACO-verified ($5{,}000$) plus the
GSM8K training split ($7{,}473$), hence $N_{\mathrm{train}}=12{,}473$.
Frontier runs use TACO-verified ($5{,}000$) plus a random $5{,}000$-prompt
subset of OpenR1-Math-46k-8192, hence $N_{\mathrm{train}}=10{,}000$.
A same-family reference pool is $100$ independent runs of that setting.
The $40$ cross-family controls use the audited setting's split, $20$ runs
each.

\begin{table}[htbp]
\centering
\caption{Compute manifest for one initial GRPO epoch ($G=16$, prompt batch $50$).
Steps are per run, $\lceil N_{\mathrm{train}}/50\rceil$.
Rollouts are the row total.
The two joint rows each describe a single training run.
Each reported joint result averages four seeds, and each seed also trains a matched Clean run of the same size.
Reference-pool and cross-family rows count every run.}
\label{tab:compute-manifest}
\small
\setlength{\tabcolsep}{3pt}
\begin{tabularx}{\linewidth}{@{}Yrrrr@{}}
\toprule
Run class & Runs & $N_{\mathrm{train}}$ & Steps / run & Rollouts \\
\midrule
Llama / Qwen2.5 joint, one run & $1$ & $12{,}473$ & $250$ & $199{,}568$ \\
Frontier joint, one run & $1$ & $10{,}000$ & $200$ & $160{,}000$ \\
Llama-setting reference pool & $100$ & $12{,}473$ & $250$ & $19{,}956{,}800$ \\
Frontier-setting reference pool & $100$ & $10{,}000$ & $200$ & $16{,}000{,}000$ \\
Cross-family controls & $40$ & --- & --- & $7{,}191{,}360$ \\
\bottomrule
\end{tabularx}
\end{table}

RQ5 continuation curves start from the step-$0$ marked checkpoint in each
setting: Llama-math uses the RQ1 single-carrier marked checkpoint; Frontier
uses the RQ2 joint marked checkpoint.
RQ4 audits the four marked checkpoints from joint training on students
from different families in the two settings (Llama and Qwen3-30B-A3B),
not the RQ5 continuation checkpoints.

\paragraph{Seeds and arms.}
RQ1 single-carrier and RQ2 joint runs each use four independent GRPO
training seeds per setting: $1$, $42$, $2026$, and $2025$.
Each seed trains a matched Clean checkpoint.  Base is the initial no-RL
student checkpoint for that setting, evaluated on the same held-out tasks;
it requires no separate training run or seed.
Tables~\ref{tab:e1}--\ref{tab:e3} subtract the paired Clean checkpoint
inside each seed, then report the mean and standard deviation of those
four differences.
The reference pool contains $50$ Clean and $50$ Sham checkpoints with
seeds disjoint from all marked suspects.
Sham arms use the decoy carriers in Appendix~\ref{app:sham-decoys}; Marked
arms use the enrolled carriers and registered code for that setting.

\paragraph{Registered codes.}
Joint Llama/Qwen2.5 training uses an explicitly registered four-sign code
on the carriers in Table~\ref{tab:llama-joint-carriers}; the same code is
used across Qwen3-30B-A3B, GPT-5.6 Sol, and GPT-6 Astra.  Frontier training
uses a separate four-sign code on Table~\ref{tab:frontier-data}, shared
across Qwen3-235B-A22B, GPT-5.6 Sol, and GPT-6 Astra in
Table~\ref{tab:e3}.
Single-carrier RQ1 uses one enrolled carrier per domain with code
$(+1)$ on the positive polarity defined in
Appendix~\ref{app:e1-carriers}.

\end{document}